\documentclass{ifacconf}

\usepackage{graphicx}      % include this line if your document contains figures
\usepackage{natbib}  
 \usepackage{amsmath,amssymb,amsfonts}
 \usepackage{xcolor}
\newtheorem{property}{Property}
\newtheorem{assumption}{Assumption}
\newtheorem{remarks}{Remarks}
\newtheorem{definition}{Definition}

\newtheorem{proof}{Proof}
\begin{document}
\begin{frontmatter}

\title{Safe Human Robot-Interaction using Switched Model Reference Admittance Control}

\author[First]{Chayan Kumar Paul}
\author{Bhabani Shankar Dey}
\author{Udayan Banerjee}
\author[Second]{Indra Narayan Kar}

\address[First]{Indian Institute of Technology Delhi, 
   New Delhi, India (e-mail: chayanpaul007@gmail.com).}

\address[Second]{Indian Institute of Technology Delhi, 
   New Delhi, India (e-mail: ink@ee.iitd.ac.in).}

\begin{abstract}
% Physical Human-Robot Interaction tasks necessitate safety requirements and compliance with human intentions. A novel switched model reference admittance control algorithm is put forth in this paper to address task-space limitations for safety. The proposition combines compliance with external forces while upholding safety constraints in a task involving human-robot interaction. The admittance control framework is used to generate the trajectory based on the switched reference models to maintain the state constraints. This enhances the robustness margin to external forces. The stability analysis of the switched system is carried out by an appropriate selection of the Common Quadratic Lyapunov function (CQLF). Finally, the proposed controller is validated for a two-link robot manipulator with external human force. The results show that the system restrains itself within the safe task-space and successfully tracks the reference trajectory while complying with humans.
Physical Human-Robot Interaction (pHRI) task involves tight coupling between safety constraints and compliance with human intentions. In this paper, a novel switched model reference admittance controller is developed to maintain compliance with the external force while upholding safety constraints in the workspace for an n-link manipulator involved in pHRI. A switched reference model is designed for the admittance controller to generate the reference trajectory within the safe workspace. The stability analysis of the switched reference model is carried out by an appropriate selection of the Common Quadratic Lyapunov Function (CQLF) so that asymptotic convergence of the trajectory tracking error is ensured. The efficacy of the proposed controller is validated in simulation on a two-link robot manipulator.
\end{abstract}

\end{frontmatter}

\section{Introduction}
With the advent of industry 4.0, it is necessary to accommodate smarter and adaptable human-robot interaction strategies, where the human operators take care of tasks that require cognitive decisions, and the non-ergonomic tasks can efficiently be carried out by the robots through automation. However, in any human-robot interaction task, the most critical factors are the \textit{safety} of the human as well as the environment and the robot's \textit{compliance} with the human operator.\\
% The most popular robot control methods utilize position control (\cite{position}) in joint-space or Cartesian Space, force control (\cite{force}), hybrid position/force control (\cite{hybrid}), and impedance/admittance-based compliance control by \cite{interaction}. Pure position control or  explicit force control methods introduce numerous stability issues due to dynamic constraints or kinematic constraints imposed by coupled dynamics. 
The interaction force is typically used to create a medium through which humans and robots communicate their intentions. To achieve compliance in physical human-robot interaction, the interaction force must be modeled to obtain a force-to-position relation. Hence, for physical Human-Robot Interaction, the admittance control strategy by \cite{admittance} is widely used where a virtual spring-mass damper system characterizes the interaction between robot and environment (or human). From the interaction force, admittance control generates a virtual trajectory that needs to be tracked by the robot. The virtual mass, damping, and stiffness can be defined with accurate knowledge of the environment. However, it becomes difficult to interact with humans due to their unparalleled ability to modulate muscle stiffness in any interaction task. Therefore, a dynamic relation is necessary by altering the virtual mass, damping, and stiffness characteristics.\\
In the mass-spring-damper system, admittance shows the relationship between the input force and the output position, whereas impedance shows the opposite. Many advanced control methodologies are also implemented, such as adaptive impedance/admittance control (\cite{adaptive}), robust impedance/admittance control (\cite{robust}), or learning-based impedance/admittance control (\cite{neural}) to obtain an efficient interaction. Recent trends in impedance/ admittance control are to modify or adapt the parameters to achieve desired interaction. \cite{admadap} used a passivity-based approach for variable admittance control. \cite{forceupd} estimated human intention from the time derivative of applied force to update the admittance parameters. Data-based methods like Fuzzy logic-based controllers by \cite{fuzzy}, or neural network-based controllers by \cite{neural} are also implemented to obtain variable admittance parameters; however, safety is often neglected while designing these controllers.\\
Since robots and humans share a common workspace, real-world deployment necessitates safety guarantees. In the interaction control framework, constraints on position and orientation in the working environment are frequently ignored in the control design, which can result in major risks or harm to the environment or the robot. The virtual trajectory generated from the admittance control may direct the robot toward an obstruction or a restricted location outside the safe task space. Control barrier functions (CBFs) as given by \cite{cbf} or Barrier Lyapunov Functions (BLFs) by \cite{ablf} are popularly used as a form of safety filter to put constraints on the states of the robot. But the issue with the BLF or CBF is that the control input becomes arbitrarily large near the safety margin, and the performance also deteriorates near the margin.\\
In this paper, a novel switched model reference admittance control framework has been developed for efficient human-robot interaction while maintaining safety. The algorithm has the following features:
\begin{itemize}
    \item provides compliance to the human when the robot operates under a safe task space.
    \item  modifies the admittance control parameters to provide safety at the outer level controller, providing more robustness.
    \item adopts a switched adaptive framework to define the update law for the parameters.
\end{itemize}
The paper is organized as follows: Section II provides an overview of the physical human-robot interaction scheme. Section III discusses the safety constraints and switched model reference strategy, followed by the stability analysis in section IV. In section V, a simulation of a two-link manipulator is presented. Finally, section VI presents the conclusions and discusses future works.
\section{Physical HRI Scheme}
% Throughout this paper, $\mathbb{R}$ denotes the set of real numbers, $\mathbb{R}^{p \times q}$ denotes set of $p \times q$ real matrices, the identity matrix in  $\mathbb{R}^{p \times p}$ is denoted by $\mathbb{I}_p$ and $\|.\|_p$ represents the Vector p-norm and corresponding equi-induced matrix norm.
\subsection{Rigid Body Dynamics and Robot Model}
The dynamic model of an n-link robot manipulator in joint space is given as described in \cite{spong},
\begin{equation}
    M_q(q)\Ddot{q}+C_q(q,\dot{q})\dot{q}+G_q(q)=\tau + \tau_{ext}  \label{eq1}
\end{equation}
where $q$, $\dot{q}$, and $\Ddot{q}\in \mathbb{R}^n$ are the vectors of joint angle, velocity, and acceleration, respectively; $M_q(q) \in \mathbb{R}^{n \times n}$ denotes the generalized inertia matrix, $C_q(q,\dot{q}) \in \mathbb{R}^{n \times n}$ denotes the Coriolis and Centripetal torque matrix, $G_q(q)$ is the gravity vector, $\tau$ is the control torque of the robot and $\tau_{ext}$ is the external torque experienced by the robot. \\
In human-robot interaction, it is much simpler for the human to follow or generate a trajectory in the Cartesian space, so the joint-space dynamics of the n-link manipulator is converted into the Cartesian space for ease of analysis.\\
The robot kinematics follows the relation:
\begin{equation}
    x =\phi(q)
\end{equation}
where $\phi:\mathbb{R}^n \rightarrow \mathbb{R}^{m}$ maps the robot's joint angles to the end-effector positions in Cartesian space, where m is the dimension of the Cartesian space and $x=[x_1, x_2...x_m]^T$ is the end-effector position in Cartesian space. \\
The Cartesian space velocities and the joint space velocities are related as
\begin{equation}
    \dot{x}=\frac{\partial \phi}{\partial q}\cdot \frac{dq}{dt}=J(q)\dot{q}
    \label{jac}
\end{equation}
where $J(q) \in \mathbb{R}^{m \times n}$ is the Jacobian of the robot manipulator.\\
The dynamics of the n-link manipulator in joint-space (\ref{eq1}) is then converted to the Cartesian space using (\ref{jac}) as follows :
\begin{equation}
    M_x(q)\Ddot{x}+C_x(q,\dot{q})\dot{x}+G_x(q)=F+F_{ext}
    \label{cart}
\end{equation}
$M_x$, $C_x$, and $G_x$ are the coefficient matrices in Cartesian space given by
\begin{align}
    M_x&=J^{-T}(q)M(q)J^{-1}(q) \\
    C_x&=J^{-T}(q)(C(q,\dot{q})-M(q)J^{-1}(q)\dot{J}(q))J^{-1}(q)\\
    G_x&=J^{-T}(q)G(q)
\end{align}

$F=J^{-T}(q)\tau$ is the force input by the robot, and $F_{ext}$ is the force experienced by the human operator, which is reflected in a relationship $\tau_{ext}=J^T(q)F_{ext}$. 
\begin{remarks}
    If the dimension of the Cartesian space and the robot manipulator's joint space is different, i.e., $m \neq n$, the pseudo-inverse of the Jacobian $(J^{-\dagger})$ is calculated in place of $J^{-1}.$
\end{remarks}
By construction, the following properties are invoked explicitly.
\begin{property}
   $
       \dot{M}_x(q)-2C_x(q,\dot{q})
 $
        is a skew-symmetric matrix (\cite{spong}). 
\end{property}
\begin{property}
    The inverse of $M_x(q)$ exits, and is positive-definite and bounded, i.e. $\|M^{-1}_x(q)\| \leq \alpha$ where $\alpha$ is a positive scalar constant (\cite{spong}). 
\end{property}
We incorporate the following logical assumptions pertaining to robot manipulators.
\begin{assumption}
 A singularity occurs when the joint velocity in joint space becomes infinite to maintain Cartesian velocity and is reflective of the break of continuity in joint space as related to Cartesian space. Singularity, mathematically expressed as $det(J)=0$, is assumed to be avoided. 
\end{assumption}
\begin{assumption}
    There exists a positive constant such that $sup(F_{ext})\leq F_{max}$, where $F_{max}$ is a scalar constant and denotes the maximum possible interaction force from the environment or the human.
\end{assumption}
% The force/position relationship is defined for the interaction using admittance control which is described below.
\subsection{Admittance Control}
To comply with the external human forces, the interaction between the robot's end-effector and the human can be modeled as a mass-spring-damper system. Let $\Psi_i=x_{r_i}-x_{d_i}$, $i=1,2,..m$, where $x_{r_i}$ is the reference trajectory generated from the human and $x_{d_i}$ is the pre-defined desired trajectory given to the robot in the Cartesian space in the $i^{th}$ direction. The dynamics for the robot manipulator rendering an impedance can be written as:
\begin{equation}
    M\Ddot{\Psi}+D\dot{\Psi}+K{\Psi}=F_{ext} \label{msd}
\end{equation}
where $\Psi \in \mathbb{R}^n$, $M$, $D$, and $K$ are the virtual mass matrix, damping matrix, and stiffness matrix, respectively. They are also known as Admittance Parameters.\\
From (\ref{msd}) and $x_{d_i}$, the upper-level admittance control provides a reference trajectory that needs to be tracked by the robot using lower-level position control. The admittance parameters are reflected on the robot by changing the apparent inertia, damping, and stiffness (\cite{interaction}). So for different environments, the admittance parameters should be modified accordingly to reflect different interaction behaviors.\\
% As the environment model is not accurately known, the interaction force can make the coupled dynamics unstable. So instead, using the admittance model, a separate spring-mass-damper system (\ref{msd}) is modeled to generate a virtual trajectory ($x_r$) that the robot will follow. By changing the admittance control parameter one can provide compliance to the interaction, however, the virtual trajectory generated online may not always be feasible as there are task space constraints or the robot's physical constraints.\\
The objective of the proposed approach is to change the admittance parameters in a way that indicates compliance when operating within a safe workspace and resistance to external force when the safety limit is crossed. To deal with this, an adaptive approach is introduced to adapt the parameters according to the reference trajectory generated. \\
The designer of MRAC (Model Reference Adaptive Control) selects a reference system that imposes the desired behavior.
% The reference model's objective is to choose the admittance parameters to comply with the human operator, given the reference state is within a pre-defined safe workspace. However, if the reference state exceeds a defined boundary in the workspace, the reference system is switched to stiffen the interaction, keeping the reference trajectory within that boundary.
The adaptive law ensures that the controlled system follows the reference system and modifies the parameter accordingly to comply with the human arm while maintaining safety constraints.
\section{Switched reference models}
\label{sec2}

% Using this information we defined a switch Model-Reference Adaptive Control strategy to move the robot freely within certain task space boundary but the end-effector becomes stiff when is reaches the boundary.\\
Consider the robot's movement due to the external human force only in one direction. In (\ref{msd}), taking $\Delta_1=\Psi$ and $\Delta_2=\dot{\Psi}$, the state space representation can be given as follows:
\begin{align}
    \dot{\Delta}_1&=\Delta_2 \nonumber\\
    \dot{\Delta}_2&=-\frac{D}{M}\Delta_2 -\frac{K}{M}\Delta_1+\frac{1}{M}U \label{syss}
\end{align} 
where $\Delta_1$ is the position of the mass and $\Delta_2$ is the velocity of the mass in the spring-mass-damper system.
In the state-space form,
\begin{equation}
    \dot{\Delta}=A\Delta + BU \label{sys}
\end{equation}
where $\Delta \in \mathbb{R}^2$ is the state vector consisting of position and velocity in the direction of applied force, $U \in \mathbb{R}$ is the control input, $A \in \mathbb{R}^{2 \times 2}$ is the unknown state transition matrix and $B \in \mathbb{R}^{2 \times 1}$ is a known input matrix.
\begin{remarks}
    If the human force is applied from multiple directions, the same framework(\ref{syss}) can be applied to obtain reference trajectory in different directions.
\end{remarks}
\begin{remarks}
 Inspired by the ability of humans to change muscle stiffness depending on the task, matrix A should not be designed a-priori; rather, it needs to be updated depending on the task.   
\end{remarks}
The second-order dynamics, as given in (\ref{msd}), reflect the relationship between output reference trajectory and external force input. When the admittance parameters, such as virtual damping and virtual spring constant, are set to higher values, the robot interaction becomes stiff, which demands higher energy. However, it is easier for a human to operate with lower parameter settings; therefore, it becomes important to encapsulate this in the control logic and switch between low and high admittance parameters based on the end-effector position to cater to the safety needs.\\
Switching the parameters to a high value instantly increases the oscillation and may lead to instability or harm the human. So a switched reference model is proposed where the system(\ref{syss}) will try to adapt the reference model parameters, and the model reference is switched instead of the actual system, depending upon the position of the reference states.
\begin{remarks}
     If one can define the adaptive update law such that asymptotic convergence of the switched system is guaranteed, reference trajectory ($x_r$) generated from the admittance control will be ensured to remain in safe region. Thus, we are defining task space safety in the admittance control parlance that will also be reflected in the robot's trajectory($x$). However, the task space safety limit needs to be defined properly as it is not absolute but depends on the robot's pre-defined desired trajectory.
\end{remarks}
The reference system is chosen to be PWA (Piece-wise Affine), and as given as:
\begin{equation}
   \dot{\Delta}_m=A_m\Delta_m+B_mr \label{delm}
\end{equation} 
where $\Delta_m \in \mathbb{R}^2$ and $r \in \mathbb{R}$ are the states and input of the reference system, respectively.\\
The reference state space($\Delta_m$) is partitioned into s polyhedral regions $\Omega_i$. The $i-th$ polyhedral region is defined by a set of $\mu_i \in \mathbb{N}$ linear inequalities, where $\mu_i$ is finite.
\begin{equation}
    \Omega_i=\left\{\Delta_m \in \mathbb{R}^2 \middle| \begin{bmatrix}
        h_{1,i}\\.\\.\\h_{\mu_i,i}
    \end{bmatrix}\begin{bmatrix}
        \Delta_m\\1
    \end{bmatrix}
    \leq_{[i]} 0\right\}
    \label{partition}
\end{equation}
where $\leq_{[i]}$ represents element-wise list of operators $<$ and $\leq$. Each row vector $h_{(.,.)} \in \mathbb{R}^{1 \times 3}$ defines a hyperplane that divides the state space into two half-spaces. 
\begin{remarks}
    Each hyperplane must be a member of one of the two adjacent partitions. This is required to obtain a state-space partition devoid of overlapping regions., i.e., $\Omega_i \cap \Omega_j = \phi,\forall i \neq j$.
\end{remarks}
The region which contains the current state vector determines which subsystem is active. The indicator function is expressed as
\begin{equation}
    \kappa_i=\begin{cases}
        1,\ \ if\ \Delta_m(t) \in \Omega_i\\
        0, \ \ otherwise. \label{indi}
    \end{cases}
\end{equation}
As the regions do not overlap, it follows that $\sum_{i=1}^s=1$ and $\kappa_i(t) \kappa_j(t)=0, i\neq j$.\\
The parameters of the reference system change according to the indicator functions $\kappa_i$ as in (\ref{indi}) and therefore given by $A_m=\sum_{i=1}^s A_{m_i}\kappa_i$ with $A_{m_i} \in \mathbb{R}^{2 \times 2} $. The virtual mass is kept constant  for all the subsystems. \\
The reference system needs to switch between models to fulfill the safety constraints and compliance with the interaction force, but the system output is a continuous smooth trajectory as it is trying to adapt to the reference system and not switching itself. The switching law can be defined using the actual system output also. \\
% \begin{remarks}
%     To get this additional freedom, an original partition $\Omega_i$ of the controlled PWA system may be divided into convex subsets $\Omega_k$ and $\Omega_l$ with $\Omega_i=\Omega_k+\Omega_l$. While the parameters of the controlled PWA system remain the same in both the regions, e.g., $A_k=A_l=A$, the reference system parameters can be chosen differently in both sets, i.e.,$\Omega_k \neq \Omega_l$. In practice, this enables switching the reference system independently from the controlled system.
% \end{remarks}
The stability of the reference system (\ref{delm}) is essential for the tracking problem. Assuming that each subsystem $\Delta_{m_i}$ is stable, it follows that there exists a symmetric, positive definite Lyapunov matrix $P_i \in \mathbb{R}^{2 \times 2}$  for every symmetric, positive definite matrix $Q_{m_i} \in \mathbb{R}^{2 \times 2}$ such that
\begin{equation}
    A^T_{m_i}P_i+P_i A_{m_i}=-Q_{m_i}\ \ \ \forall\ i=1,2,..s. \label{lyapa}
\end{equation}
Hence, $V_i=x^TP_ix$ is a quadratic Lyapunov function for the i-th subsystem of the reference system.
\begin{definition}
    If all subsystems in (\ref{delm}) satisfy (\ref{lyapa}) for a common matrix P, i.e.,$P_i=P, \ \forall i,$ then $V=x^TPx$ is referred to as Common Quadratic Lyapunov Function (CQLF) or Common Lyapunov Function (CLF) (\cite{switching}). \label{def} 
\end{definition}
The common Lyapunov function approach is sometimes viewed as more conservative than the multiple Lyapunov function method, although it also offers its own benefits.
\begin{enumerate}
    \item A common Lyapunov function, if it exists, is easier to obtain.
    \item If there exists a CQLF for the reference system (\ref{delm}), then the system is asymptotically stable for any arbitrary switching.(\cite{switch})
\end{enumerate}
Considering the Direct MRAC case, the control input to the system takes the form
\begin{equation}
    U(t)=K_x(t)\Delta(t) + K_r(t)r(t) \label{ctrl}
\end{equation}
with estimated control gains $K_x \in \mathbb{R}^{1 \times 2}$, and  $K_r=1$, as $B_m=B$ is considered.
\begin{remarks}
   It is necessary to simultaneously activate a set of control gains whenever the system's parameters change. Therefore, the gain is controlled by the same switching functions, $\kappa_i$, i.e., $K_x=\sum_{i=1}^s K_{x_i}\kappa_i$.
\end{remarks}
Controlling the system(\ref{sys}) with the control input as given in (\ref{ctrl}) yields the closed-loop system
\begin{align}
    \dot{\Delta} &= A \Delta +B K_x(t)  \Delta +Br\\
    &=\sum_{i=1}^s(A+BK_{x}\kappa_i)\Delta +Br
\end{align}
which should be equal to (\ref{delm}) for suitably chosen control gains. 
\begin{assumption}
   There must be nominal control gains $K_{x_i}^*$ that meet the matching conditions for the tracking issue to be feasible.
\begin{equation}
    A_{m_i}=A+B K_{x_i}^*\ \ \ \forall\ i=1,2..s. \label{matching}
\end{equation}
\end{assumption}
A well-known advantage of the MRAC is that the control gains need not converge to the nominal control gains. The gain update rule must be specified to track the reference trajectory so that, even if we change the model reference to be very stiff, our actual system will only vary the gains to track the reference system output.
\section{Stability Analysis}
The objective of the proposition is to design the algorithm such that the error between the system output trajectory and that of the switched Model Reference should converge to zero asymptotically. For the convergence guarantee, the stability analysis of overall switched system is important to study. This is presented in the form of following theorem. 
\begin{thm}
\label{theorem}
Consider the system as given by (\ref{sys}) and the switched model reference (\ref{delm}). If there exists a positive definite common quadratic Lyapunov matrix P, then the actual system (\ref{sys}) output will converge to the switched model reference (\ref{delm}) asymptotically for any arbitrary switching with parameter adaptive gains as given
\begin{equation}
    \dot{K}_{x_i}= - \Gamma_iB^TPex^T \kappa_i \label{upd}
\end{equation}
where $\Gamma_i \in \mathbb{R}^{2 \times 2}$ is a positive diagonal matrix, and $P \in \mathbb{R}^{2 \times 2}$ is the symmetric, positive definite common quadratic Lyapunov matrix as defined in (\ref{def}).
\end{thm}
\begin{proof}
    The error between the current control gains and the nominal gains is defined as $\Tilde{K}_{x_i}=K_{x_i}-K_{x_i}^*$. Rewriting the closed-loop system dynamics using the reference system and the gain error:
\begin{align}
    \dot{\Delta}&=\sum_{i=1}^s (A+BK_{x_i} \kappa_i)\Delta +Br \nonumber \\
    &=\sum_{i=1}^s(A+B K_{x_i}^* \kappa_i)\Delta +Br +\sum_{i=1}^s \kappa_i B \Tilde{K}_{x_i}\Delta  \nonumber \\
    &= A_m \Delta+ Br+\sum_{i=1}^s \kappa_i B \Tilde{K}_{x_i}\Delta  \label{cls}
\end{align}
From (\ref{cls}) and (\ref{delm}), the dynamics of the tracking error, e=$\Delta-\Delta_m$, is given as
\begin{equation}
    \dot{e}=A_m e + \sum_{i=1}^s \kappa_i B \Tilde{K}_{x_i}\Delta \label{error}
\end{equation}
If the control gains $K_{x_i}$ take the nominal values defined in (\ref{matching}), then the error dynamics defined in (\ref{error}) reduced to $\dot{e}=A_m e$, which indicates that the tracking error e converges to zero exponentially for stable reference system.\\
Consider the following Lyapunov Candidate function,
\begin{align}
    V&= \frac{1}{2}e^TPe+\frac{1}{2} \sum_{i=1}^s\bigg( tr(\Tilde{K}_{x_i}^T \Gamma_i^{-1}\Tilde{K}_{x_i})\bigg) \label{lyap}
\end{align}
where $\Gamma_i$ is a diagonal positive definite gain matrix.
The candidate Lyapunov function is a positive definite, radially unbounded function. The time derivative of (\ref{lyap}) is
\begin{align}
    \dot{V} &= e^T\bigg(\frac{1}{2}\sum_{i=1}^s (A_{m_i}^TP+PA_{m_i})\kappa_i\bigg)e+\sum_{i=1}^s(\kappa_i e^TPB\Tilde{K}_{x_i}\Delta ) \nonumber \\
    &+\sum_{i=1}^s tr(\Tilde{K}_{x_i}^T \Gamma_i^{-1}\dot{\Tilde{K}}_{x_i})\\
    &=e^T\bigg(-\frac{1}{2}\sum_{i=1}^sQ_{m_i}\kappa_i\bigg)e+\sum_{i=1}^s(\kappa_i e^TPB\Tilde{K}_{x_i}\Delta ) \nonumber \\
    &+\sum_{i=1}^s tr(\Tilde{K}_{x_i}^T \Gamma_i^{-1}\dot{K}_{x_i}) \label{vdot}
\end{align}
If the parameter update law ($\dot{K}_x$) is chosen as (\ref{upd}), then the time derivative becomes negative semi-definite as follows.
\begin{equation}
    \dot{V}=e^T\bigg(-\frac{1}{2}\sum_{i=1}^sQ_{m_i}\kappa_i\bigg)e \leq 0
\end{equation}
From La-salle's invariance principle (\cite{invariance}), the asymptotic convergence of the error between the reference model output trajectory and the admittance controller output trajectory can be proved. 
\end{proof}

\section{Results and Discussion}
 The proposed method incorporates the following broad steps.
\begin{itemize}
    \item An external force in the XY-plane is applied to a two-link robot manipulator to validate the results.
    \item Accordingly, the virtual trajectory($\Delta_1$)  in the X and Y directions are generated online, maintaining the safety constraints using switched reference admittance control.
    \item  Consequently, the reference trajectory is generated, which is used for lower-level control design.
    \item A standard feedback linearizing control structure in Cartesian space is designed at the lower level to achieve the desired trajectory tracking.
\end{itemize}

\subsection{Robot manipulator control}
Consider a two-link robot manipulator in the horizontal plane for simulation purposes. Based on (\ref{eq1}), the dynamic parameters of the robot are given as:
\begin{equation}
    M_q(q) = 
 \begin{bmatrix}
        p_1+ 2p_2+(m_1+m_2)l_1^2 & p_1+p_2\\
        p_1+p_2 & p_1
    \end{bmatrix}
\end{equation}
\begin{equation}
    C_q(q,\dot{q}) =\begin{bmatrix}
        -2m_2l_1l_2s_2\dot{q}_2 & -m_2l_1l_2s_2\\
        m_2l_1l_2s_2\dot{q}_1 & 0
    \end{bmatrix} 
\end{equation}
\begin{equation}
   G_q(q) =\begin{bmatrix}
        m_2l_2g c_{12} +(m_1+m_2)l_1g c_1\\
        m_2l_2g c_{12}
    \end{bmatrix}  
\end{equation}

where, $p_1=m_2l_2^2$, $p_2=m_2l_1l_2c_2$, $c_i=cos(q_i)$, 
 $c_{ij}=cos(q_i+q_j)$, $s_i=sin(q_i)$, $s_{ij}=sin(q_i+q_j)\  \forall i$, $j \in \mathbb{N}.$ $m_1=1.5$ kg, $m_2=1$ kg are the mass of links 1 and 2, respectively. The length of the links are $l_1=l_2=0.85$ m, and $g=9.81$ $m/s^2$ is the acceleration due to gravity.\\
The Jacobian of the robot manipulator is given as
\begin{equation}
    J(q)=\begin{bmatrix}
        -l_1s_1-l_2s_{12} & -l_2s_{12}\\
        l_1c_1+l_2c_{12}  &  l_2c_{12}
    \end{bmatrix}
\end{equation}
 The reference trajectory for the robot to follow can be generated from the admittance control framework considering the desired location is at the origin
\begin{equation}
    x_{r}(t)=\begin{bmatrix}
        \Delta_{1x}\\ \Delta_{1y}
    \end{bmatrix}
\end{equation}
where $\Delta_{1x}$ and $\Delta_{1y}$ are the reference trajectory generated in X and Y direction, respectively.\\
A PD controller is chosen for trajectory tracking as the lower-level position controller of the robot in the task space. The idea is to determine the force needed in task space to track the desired reference trajectory and then convert the Cartesian force to the joint input torque.\\
The control input force can be defined from Eq. (\ref{cart}) using Feedback Linearization as follows:
\begin{align}
    F=M_xa+C_x(q,\dot{q})\dot{x}+G_x(q)-F_{ext}
    \label{controleff}
\end{align}
This results in a closed loop unit mass dynamics as $\Ddot{x}= a$, where a is a virtual acceleration. The virtual acceleration needs to be designed so that the error between the reference trajectory and the robot trajectory goes to zero. The error($e=x_r-x$) is used to design the virtual acceleration as
\begin{align}
    a=\Ddot{x}_r+K_d\dot{e}+K_p e
\end{align}
where $K_d, K_p$ are the PD controller gains. 
\subsection{Switched Admittance Control}
As described in Section \ref{sec2}, the reference system is chosen using (\ref{delm}) with the following subsystem matrices
\begin{equation*}
    A_{m_1}=\begin{bmatrix}
        0 & 1\\ -5 & -9
    \end{bmatrix}
    \ \ \ B_{m_1}=\begin{bmatrix}
        0 \\ 1
    \end{bmatrix}
\end{equation*}
\begin{equation*}
    A_{m_2}=\begin{bmatrix}
        0 & 1\\ -20 & -25
    \end{bmatrix}
    \ \ \ B_{m_2}=\begin{bmatrix}
        0 \\ 1
    \end{bmatrix}
\end{equation*}

    The maximum value of the external force input is chosen at $F_{max}=20N$, which is acceptable from an engineering point of view.\\
    The admittance parameter values are chosen small for the first sub-system maintaining the compliance property. For the second subsystem, the stiffness is chosen such that upon applying the maximum external force, the steady state value of the reference trajectory will remain within the defined bound, and the damping is also chosen accordingly.

The safety constraint is put as the reference position output of the admittance control should be within a prescribed value and described according to Eq. \ref{partition}.\\
The region partitions are given as

\begin{equation}
\Omega=\begin{cases}
    \Omega_1,\ if\ \{ \Delta_{m_1} \in \mathbb{R}\bigg|\  |\Delta_{m_1}| \leq (1-\phi)\},\\
    \Omega_2 ,\ if\ \{ \Delta_{m_1} \in \mathbb{R}\bigg|\  |\Delta_{m_1}| > (1-\phi)\}.
\end{cases}
\end{equation}
where $\Delta_{m_1}$ represents the position of the reference model output and $\phi$ is a tolerance value based on the system's hardware. The switching occurs when the state crosses the task-space safety limit, so it is inevitable to confine the trajectory within the defined boundary. Hence a safe threshold ($1-\phi$) is defined for the simulation with $\phi = 0.2\%$ of the actual safety boundary.\\
The $Q_{m_i}$ matrices are chosen by the user, and accordingly, the common quadratic Lyapunov matrix, $P$ is obtained by solving the matrix inequalities as described in Theorem \ref{theorem},
\begin{equation*}
    P=\begin{bmatrix}
        8.16 & 2.22 \\ 2.22 & 3.90
    \end{bmatrix},
    \ \ 
    Q_{m_i}= \begin{bmatrix}
        1 & 0\\ 0 & 1
    \end{bmatrix}\ \forall\ i.
\end{equation*}
The update gain for the different switching regions is chosen as
\begin{equation*}
    \Gamma_1=\begin{bmatrix}
        200 & 0 \\ 0 & 200
    \end{bmatrix}\ \ \ 
    \Gamma_2=\begin{bmatrix}
        1000 & 0 \\ 0 & 1000
    \end{bmatrix}
\end{equation*}
\begin{remarks}
    The gain in region 2($\Omega_2$) needs to be larger than the gain in region 1($\Omega_1$) for all time to follow the safety constraint strictly.
\end{remarks} 
An adaptive update law for the admittance parameters is obtained from the adaptation law as given in Eq. (\ref{upd}).

The force input to the reference system is given as
\begin{equation}
    F_{ext}=\begin{bmatrix}
        F_x\\F_y
    \end{bmatrix}=\begin{bmatrix}
        7.5sin(0.5t)\\7.5cos(0.5t)
    \end{bmatrix}
\end{equation}
where $F_x$ and $F_y$ are the external forces acting on the manipulator's end-effector from the X and Y directions, respectively. The frequency of the external force is kept low (within 1.2Hz) so that the stability is maintained for the overall system as described in \cite{stability}.\\
\begin{figure}[htbp]
  \centering
  \includegraphics[width=0.9\linewidth]{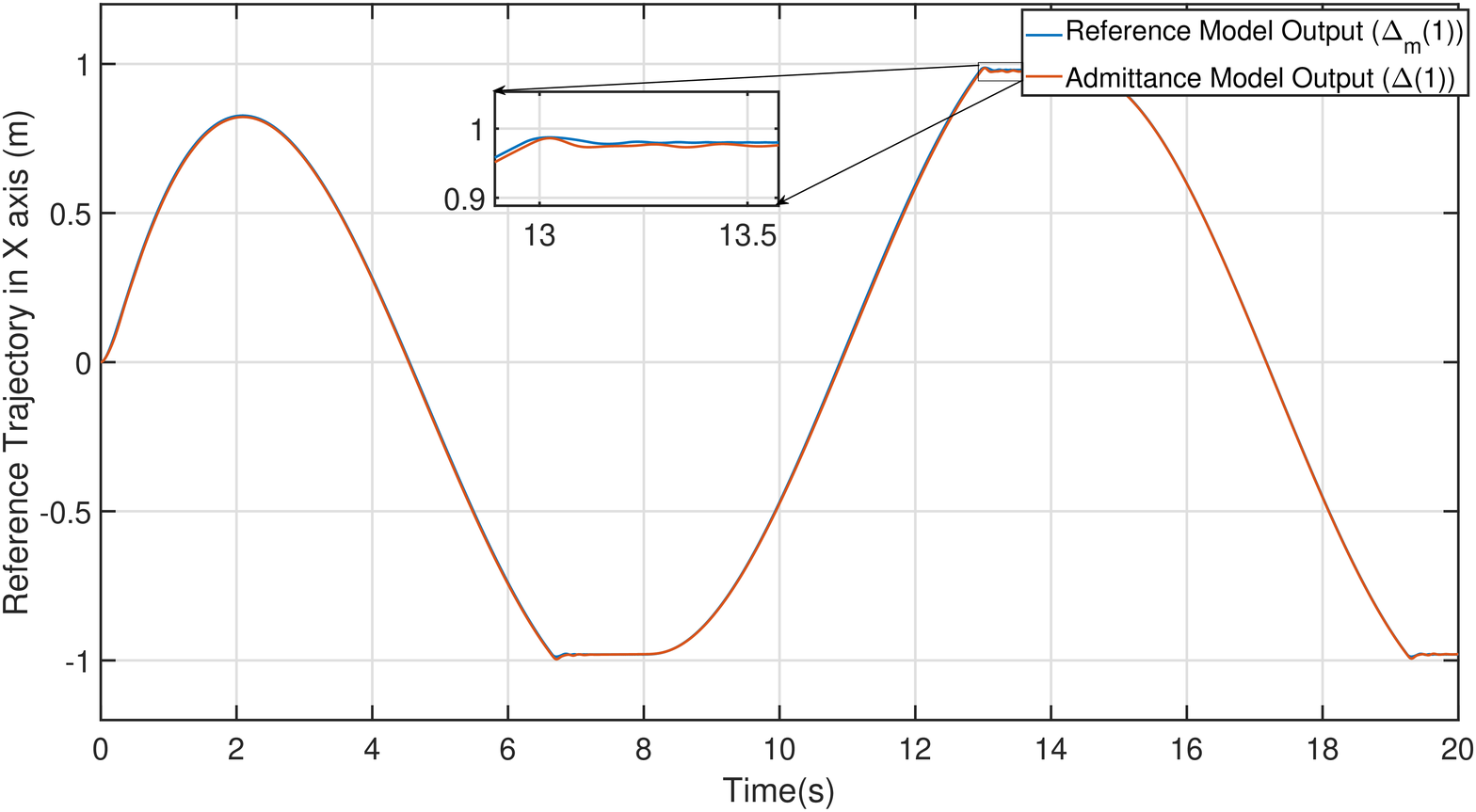}
  \caption{Admittance Controller output}
  \label{admout}
\end{figure}
From Fig. \ref{admout}, one can observe the compliance property of the robot when the external force is low. However, when the force is increased, the reference system switches to high admittance parameters model to keep the virtual trajectory($\Delta_{1x}$) within the safety limit, i.e., 1m. Similarly, in the Y direction, the trajectory($\Delta_{1y}$) is obtained using the safety limit.\\
    One can also observe that the system is becoming stiff upon reaching the safety limit without compromising performance.
\begin{figure}[htbp]
  \centering
  \includegraphics[width=0.9\linewidth]{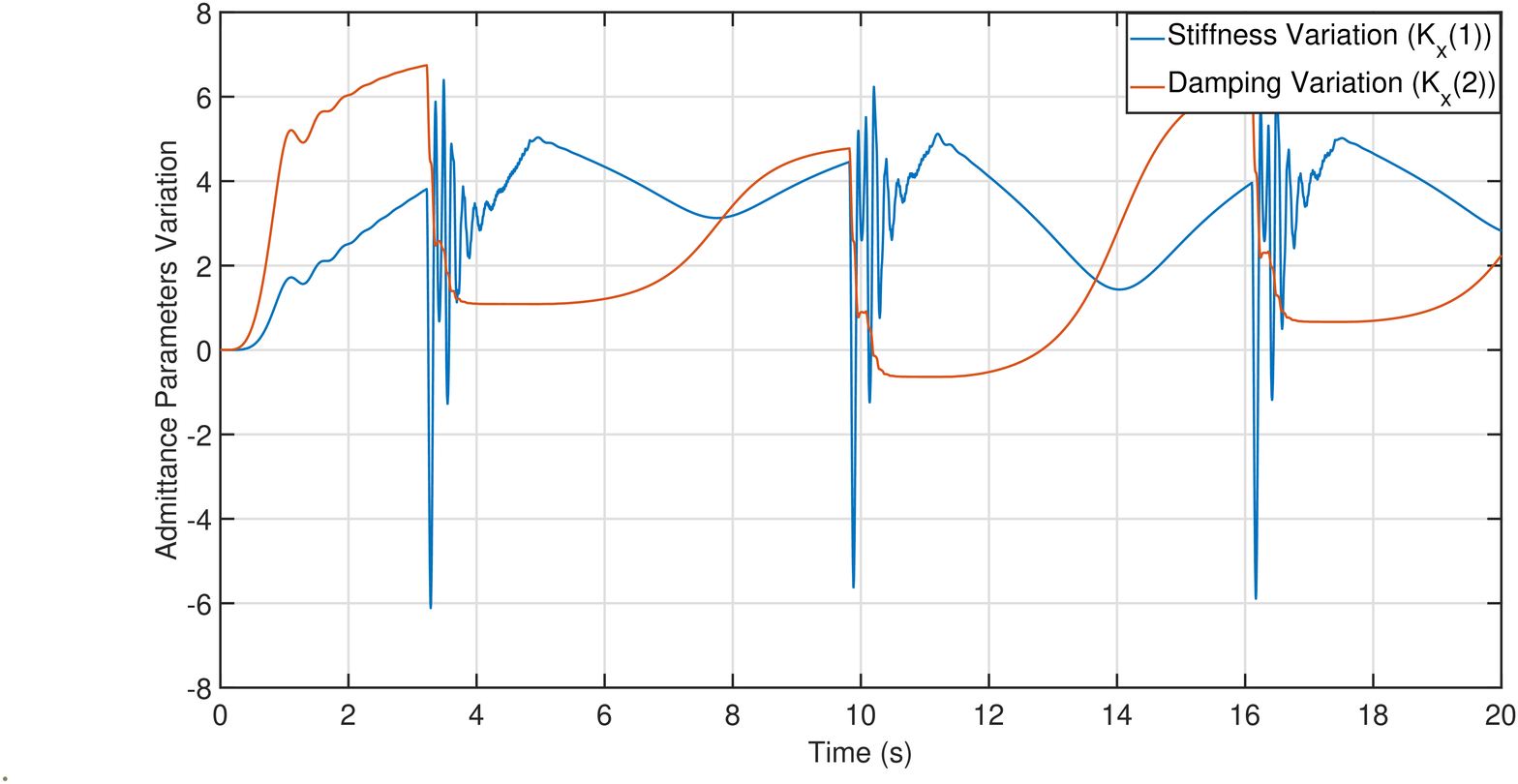}
  \caption{Admittance Parameter Variations (Stiffness($K_x(1)$)         Damping($K_x(2)$))}
  \label{param}
\end{figure}
The admittance parameters are updated following Eq. \ref{upd}, and the parameters variations are plotted in Fig. \ref{param}. The stability analysis ensures no parameter drift and the oscillations are decreasing, as one can infer from Fig. \ref{param}. Our future task will incorporate the issue of oscillations in the parameters.\\
\begin{figure}[htbp]
  \centering
  \includegraphics[width=0.9\linewidth]{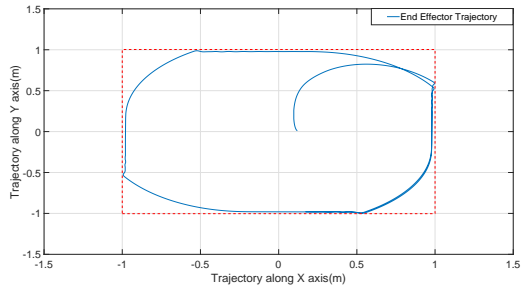}
  \caption{End-Effector Trajectory in the X-Y plane}
  \label{traj}
\end{figure}
The reference trajectory obtained from the admittance control model is tracked by the two-link manipulator following Eq.\ref{controleff}. The end-effector trajectory and the safety margin are given in Fig. \ref{traj}. \\
\begin{figure}[htbp]
  \centering
  \includegraphics[width=0.9\linewidth]{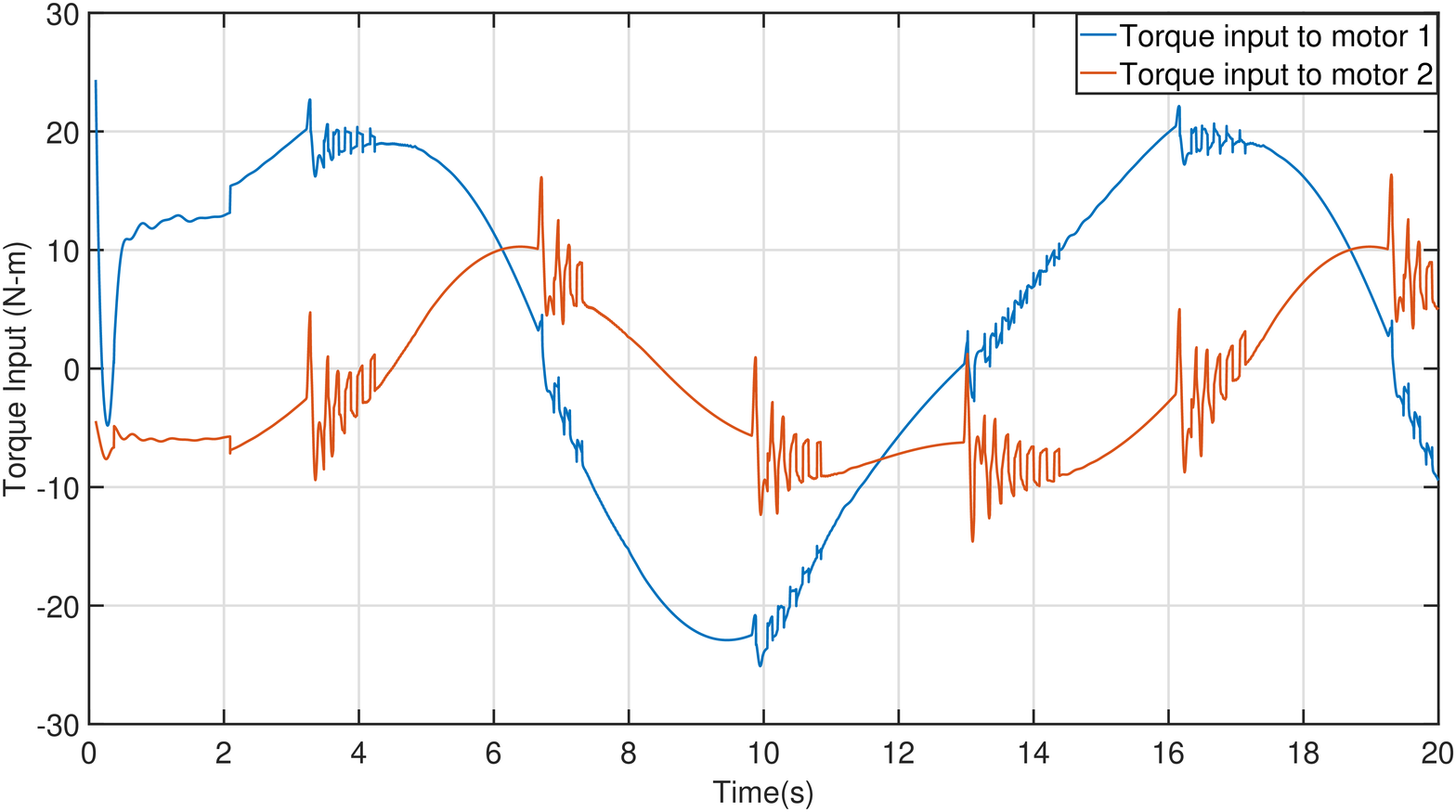}
  \caption{Control Input to the Two-Link Manipulator Motors ($\tau$)}
  \label{ctrli}
\end{figure}
Usually, to confine the robot within certain constraints, a high value of control input is required near the boundary. The control effort as shown in Figure. \ref{ctrli} is also within the range of 30N, where the maximum force exerted by the human is 7.5N, which the motors can easily deliver.
\section{Conclusion}
A novel switched admittance controller framework for safe human-robot collaboration with human force compliance is proposed. The state-dependent switched reference model is introduced to update the admittance control parameters. The user chooses parameter values of reference models and based on that, a common quadratic Lyapunov function (CQLF) is obtained to establish the stability of the switched system. Experimental results are also provided to illustrate the successful performance of the controller. In the future, the proposed controller will be implemented on a hardware setup to enhance the feasibility of the results.

\bibliography{ref} 

\vspace{12pt}

\end{document}